\documentclass{article}

\usepackage{microtype}
\usepackage{graphicx}
\usepackage{booktabs}
\usepackage{subcaption}   

\usepackage{amsmath,amssymb,amsfonts}
\usepackage{mathtools}
\usepackage{amsthm}

\usepackage{algorithm}
\usepackage{algpseudocode}

\usepackage[most]{tcolorbox}
\usepackage{multirow}
\usepackage{placeins}
\usepackage[normalem]{ulem}

\usepackage{hyperref}     
\usepackage[capitalize,noabbrev]{cleveref} 

\usepackage[preprint]{icml2026}    

\usepackage[textsize=tiny]{todonotes}
\usepackage{enumitem}

\usepackage[table]{xcolor}   
\usepackage{colortbl}

\theoremstyle{plain}
\newtheorem{theorem}{Theorem}[section]
\newtheorem{proposition}[theorem]{Proposition}

\theoremstyle{definition}

\newtheorem{assumption}[theorem]{Assumption}
\theoremstyle{remark}

\usepackage{algorithm}
\usepackage{algpseudocode}
\usepackage{mathtools}
\usepackage[most]{tcolorbox}
\usepackage{booktabs}
\usepackage[normalem]{ulem}

\usepackage{graphicx} 
\usepackage{amsthm}
\usepackage{placeins}
\usepackage{multirow} 
\usepackage{hyperref}

\usepackage{amsmath, amssymb, amsfonts, amsthm}

\usepackage[textsize=tiny]{todonotes}

\icmltitlerunning{UniHash: Unifying Pointwise and Pairwise Hashing Paradigms}

\begin{document}

\twocolumn[
\icmltitle{UniHash: Unifying Pointwise and Pairwise Hashing Paradigms}

\icmlsetsymbol{equal}{*}
\begin{icmlauthorlist}
\icmlauthor{Xiaoxu Ma}{scut,gatech}
\icmlauthor{Runhao Li}{ntu}
\icmlauthor{Xiangbo Zhang}{gatech}
\icmlauthor{Zhenyu Weng}{scut}
\end{icmlauthorlist}

\icmlaffiliation{scut}{South China University of Technology}
\icmlaffiliation{gatech}{Georgia Institute of Technology}
\icmlaffiliation{ntu}{Nanyang Technological University}

\icmlcorrespondingauthor{Zhenyu Weng}{wzytumbler@gmail.com}

\vskip 0.3in

\vskip 0.3in
]

\printAffiliationsAndNotice{}

\begin{abstract}
Effective retrieval across both seen and unseen categories is crucial for modern image retrieval systems. Retrieval on seen categories ensures precise recognition of known classes, while retrieval on unseen categories promotes generalization to novel classes with limited supervision. However, most existing deep hashing methods are confined to a single training paradigm, either pointwise or pairwise, where the former excels on seen categories and the latter generalizes better to unseen ones. To overcome this limitation, we propose Unified Hashing (UniHash), a dual-branch framework that unifies the strengths of both paradigms to achieve balanced retrieval performance across seen and unseen categories. UniHash consists of two complementary branches: a center-based branch following the pointwise paradigm and a pairwise branch following the pairwise paradigm. A novel hash code learning method is introduced to enable bidirectional knowledge transfer between branches, improving hash code discriminability and generalization. It employs a mutual learning loss to align hash representations and introduces a Split-Merge Mixture of Hash Experts (SM-MoH) module to enhance cross-branch exchange of hash representations. Theoretical analysis substantiates the effectiveness of UniHash, and extensive experiments on CIFAR-10, MSCOCO, and ImageNet demonstrate that UniHash consistently achieves state-of-the-art performance in both seen and unseen image retrieval.
\end{abstract}
\section{Introduction}


Hashing has become a prominent solution for large-scale image retrieval due to its efficiency in both computation and storage \citep{li2011hashing,jiang2025online,cao2025deep,liang2024self,kong2024mitigating,wu2023forb}. It maps high-dimensional visual features into compact binary codes while preserving semantic similarity in the Hamming space \citep{cao2018deep,yang2018adversarial,qin2018gph,jegou2008hamming}. Recent deep hashing methods have achieved state-of-the-art performance by jointly optimizing feature extraction and hash code learning in an end-to-end fashion \citep{yuan2020central,hoe2021one,wang2023deep,he2024flexible,liu2016deep,li2015feature,cao2017hashnet,wang2017deep,su2018greedy,fan2020deep}. Despite these advances, real-world retrieval systems often face queries from both seen categories encountered during training and unseen categories that emerge after deployment. Achieving balanced performance across two settings is crucial for developing reliable retrieval systems, a challenge that existing hashing methods have overlooked.


\begin{figure*}[t]
    \centering
    \includegraphics[width=\linewidth]{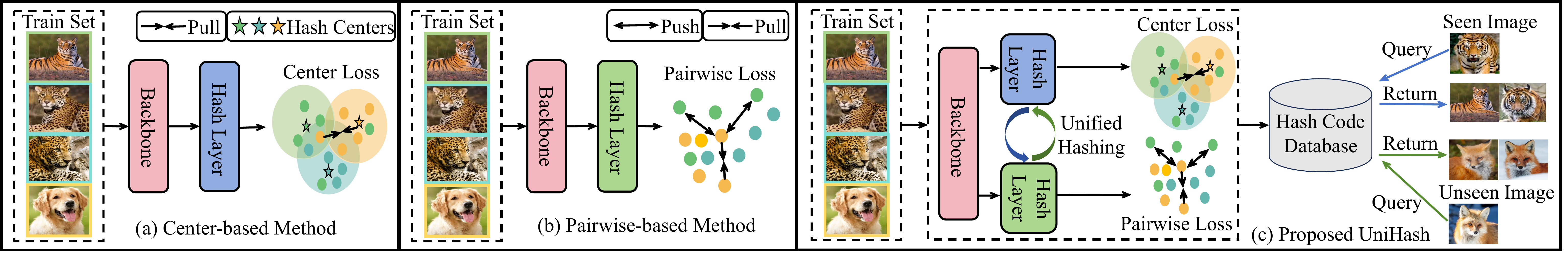}
    \caption{Comparison of existing deep hashing paradigms with the proposed Unified Hashing (UniHash). The figure depicts the evolution from traditional center-based and pairwise hashing to the proposed unified hashing paradigm, which fosters complementary supervision between branches and enhances retrieval performance on both seen and unseen image categories.}
    \label{fig:method_comparison}
\end{figure*}

Deep supervised hashing methods can generally be divided into two paradigms: pointwise and pairwise.
As shown in Fig.~\ref{fig:method_comparison}(a), recent pointwise methods ($e.g.,$ center-based methods)\citep{yuan2020central,hoe2021one,wang2023deep} introduce fixed or learnable hash centers that explicitly represent category prototypes. These methods establish representative hash centers for each category and train the network to align the hash codes of individual samples with their corresponding centers. Such pointwise supervision effectively enforces intra-class compactness \citep{wei2022weakly,xuan2021intra,kobayashi2021t,wang2023improving}, yielding superior retrieval performance on seen categories. However, their strong reliance on predefined centers limits their generalization to unseen categories~\cite{shen2025empirical}. In contrast, as shown in Fig.~\ref{fig:method_comparison}(b), pairwise methods learn relative similarities among samples. Specifically, binary pairwise models \citep{liu2016deep,li2015feature,cao2017hashnet} aim to pull similar pairs closer and push dissimilar pairs apart in the Hamming space. Triplet-based extensions \citep{wang2017deep,deng2018triplet} introduce an anchor–positive–negative structure to capture relative ranking information, while listwise variants \citep{liang2021deep} further optimize global ranking metrics by jointly considering multiple items. This relational supervision flexibly preserves semantic consistency through feature similarity, thereby enhancing generalization to unseen categories.

As existing deep hashing methods are trained under a single supervision paradigm, either pointwise or pairwise, they struggle to achieve balanced image retrieval performance across seen and unseen categories~\cite{shen2025empirical}. To address this challenge, we propose Unified Hashing (UniHash), a dual-branch framework that unifies the strengths of both paradigms to achieve balanced image retrieval performance across seen and unseen categories. As illustrated in Fig.~\ref{fig:method_comparison}(c), UniHash consists of two complementary branches: a center-based branch corresponding to the pointwise paradigm and a pairwise branch corresponding to the pairwise paradigm. To enhance both the discriminability and generalization of hash codes, we introduce a novel hash code learning method that facilitates bidirectional knowledge exchange between two branches. Specifically, a mutual learning loss is employed to align hash code representations between two branches. Inspired by the Mixture of Experts (MoE) framework~\citep{shazeer2017outrageously,chen2023adamv,chen2022towards,riquelme2021scaling}, a Split-Merge Mixture of Hash Experts (SM-MoH) module is designed to promote cross-branch interaction in hash representation learning. Through this design, the center-based branch acquires relational cues from the pairwise branch, while the pairwise branch benefits from the compact semantic structure maintained by the center-based branch. Furthermore, we theoretically show that UniHash alleviates paradigm-specific discrepancies by enforcing the consistency between the two learning branches.

In summary, our main contributions are as follows:
\begin{itemize}[label=$\bullet$]

    \item We propose Unified Hashing (UniHash), a novel dual-branch framework that unifies the pointwise and pairwise training paradigms to integrate their complementary strengths.

    \item We introduce a hash code learning method that promotes cross-branch interaction and alignment of hash representations for more discriminative and generalizable hash codes.

    \item Theoretical analysis validates the effectiveness of UniHash, and extensive experiments demonstrate its state-of-the-art performance across both seen and unseen retrieval settings.

\end{itemize}

\section{Related Works}

\noindent\textbf{Pointwise/Pairwise Training Paradigms.} 
From the perspective of training paradigms, deep supervised hashing methods can be broadly categorized into pairwise and pointwise approaches. Early works such as Deep Hashing~\citep{Liong2015Deep} and HashNet~\citep{cao2017hashnet} follow the pairwise paradigm, which learns hash functions by modeling pairwise or triplet relationships. In contrast, pointwise methods such as CSQ~\citep{yuan2020central}, OrthoHash~\citep{hoe2021one}, and MDSH~\citep{wang2023deep} define category specific hash centers, either pre-defined or learnable, to cluster semantically similar samples around a common center while pushing apart dissimilar ones. Recent studies indicate that pointwise methods perform better on seen categories, whereas pairwise methods generalize more effectively to unseen categories~\cite{shen2025empirical}. In this work, we aim to unify the strengths of both paradigms to achieve balanced retrieval performance across seen and unseen categories.


\noindent\textbf{Deep Mutual Learning (DML).} DML~\citep{zhang2018deep} is initially introduced as a collaborative training strategy in which multiple networks learn jointly by aligning their predicted class probabilities to improve generalization. Building on this idea, subsequent work extended DML to visual object tracking~\citep{zhao2021deep}, where lightweight networks mutually supervise each other during offline training to enhance feature representation and tracking accuracy. Unlike prior approaches~\citep{zhang2018deep,guo2022online,zhao2021novel,zhao2021deep} that train multiple networks under identical objectives, we enable two branches with distinct supervision paradigms to iteratively guide each other for joint hash function optimization.

\noindent\textbf{Mixture of Experts (MoE).} The MoE framework~\citep{shazeer2017outrageously} has recently been applied to large-scale neural networks, leveraging multiple specialized experts and a top-k gating mechanism for efficient computation. It has since been extended to multi-task learning in computer vision, such as AdaMV-MoE~\citep{chen2023adamv}, which introduces task-specific routing and adaptive expert selection for diverse recognition tasks. Inspired by the effectiveness of MoE in modeling task-dependent information~\citep{shazeer2017outrageously,chen2023adamv,chen2022towards,riquelme2021scaling}, we propose the Split-Merge Mixture of Hash Experts (SM-MoH) module within our framework. Unlike standard MoE, SM-MoH adopts a split-merge routing mechanism: branch-specific gates independently select experts to capture distinct supervisory cues, while shared experts merge their outputs to ensure transformation consistency across branches. This design enhances cross-branch interaction in hash representation and improves the overall quality of hash codes.

\begin{figure*}[t]
    \centering
    \includegraphics[width=0.95\textwidth]{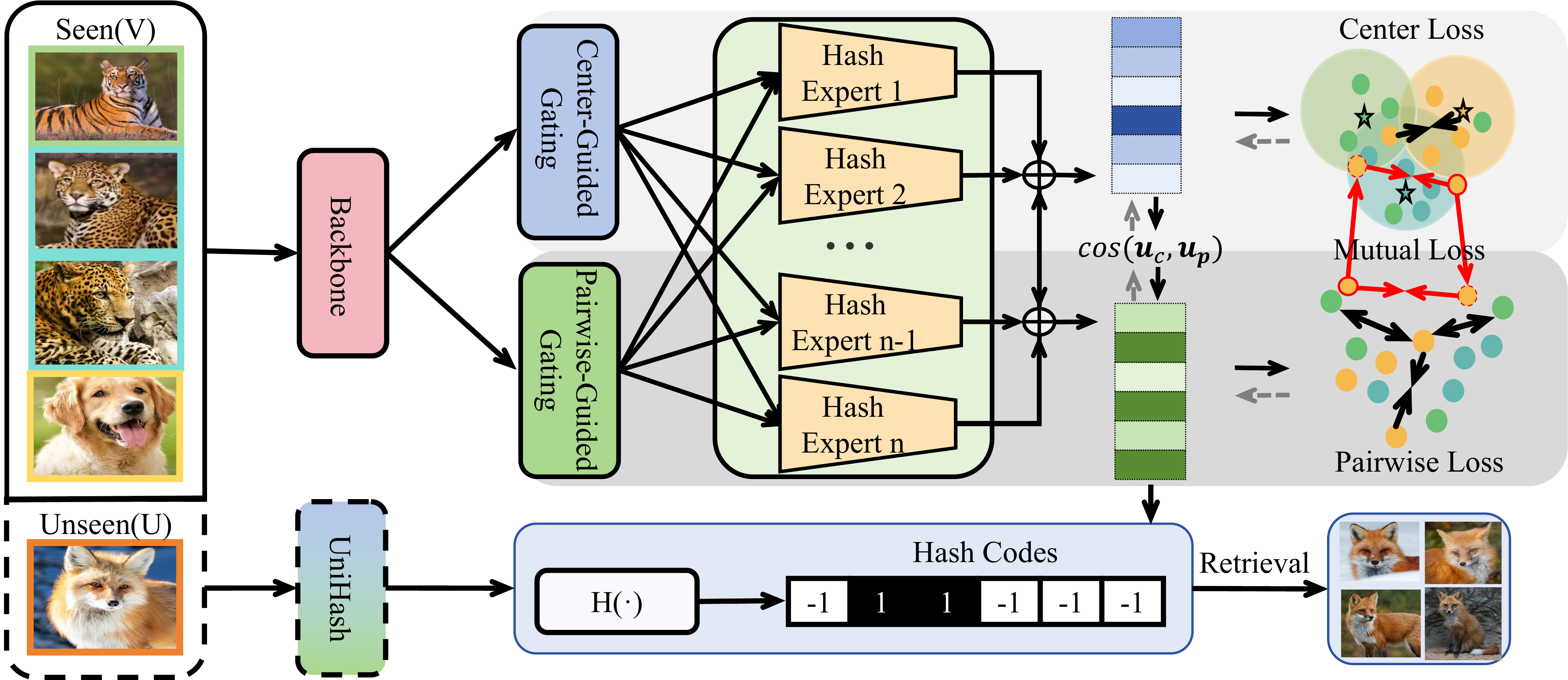}
    \caption{Overview of the proposed Unified Hashing (UniHash) framework. A deep neural network extracts image features, which are fed into two parallel branches producing center-based codes $u_c$ and pairwise codes $u_p$. During training, the center-based branch learns relational cues from the pairwise branch, while the pairwise branch benefits from the semantic compactness of the center-based branch, leading to more discriminative and generalizable hash codes.}
    \label{fig:overall_structure}
\vspace{-3mm}
\end{figure*}
\section{Methodology}
We propose Unified Hashing (UniHash), a framework that unifies the complementary strengths of pairwise and center-based methods to achieve superior retrieval performance across both seen and unseen categories. We first present the overall framework in Sec.~\ref{subsec:framewrok}, and then provide a theoretical analysis of our methodology in Sec.~\ref{sec:theory}.

\subsection{UniHash Framework}
\label{subsec:framewrok}
Given an image dataset $\mathcal{X}=\{x_i\}_{i=1}^N$ with labels 
$\mathcal{Y}=\{y_i\}_{i=1}^N$, deep hashing learns a mapping
\[
H: \mathcal{X} \rightarrow \{-1,1\}^q,
\]
encoding each image $x_i$ into a $q$-bit code 
$b_i = \mathrm{sign}(u_i)$. In our method, a backbone $\phi(\cdot)$ (ResNet-50 typically) extracts features 
$v_i = \phi(x_i)$, followed by two parallel branches that produce 
continuous hash codes $u_i^c,\,u_i^p \in (-1,1)^q$ for the 
center-based and pairwise branches, respectively.

As illustrated in Fig.~\ref{fig:overall_structure}, UniHash comprises two main components: (1) a mutual learning loss that aligns hash representations between the two branches, and (2) a Split-Merge Mixture of Hash Experts (SM-MoH) module that enhances cross-branch interaction in hash representation learning. The details of these components are presented in the following subsections.

\begin{figure*}[htbp]
    \centering
    \includegraphics[width=1.0\textwidth]{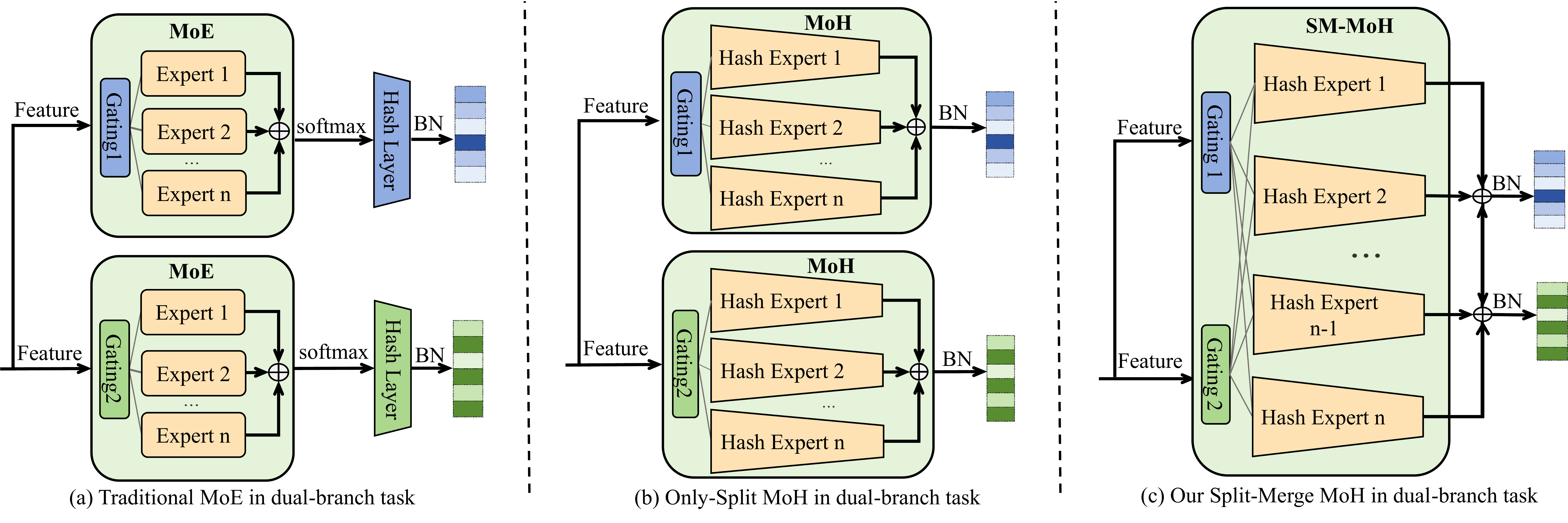}
    \caption{Comparison of expert-based hashing architectures. (a) Conventional MoE with separate experts per branch. (b) Only-Split MoH with independent gates and experts generating continuous hash codes. (c) Our proposed MoH with independent gates and shared experts.}
    \label{fig:MoH}

\end{figure*}
\vspace{-1mm}

\subsubsection{Mutual Learning Loss}

In this paper, mutual learning loss refers to the composite objective combining the center-based loss $L_C$, the pairwise loss $L_P$, and the cross-branch alignment term $L_M$:
\begin{equation}
L=\lambda_1L_C+\lambda_2L_P+\lambda_3L_M,
\end{equation}
where $\lambda_1$, $\lambda_2$, and $\lambda_3$ are the hyerparameters.

\noindent\textbf{Center-based Supervision.} Each semantic class $c\!\in\!\{1,\dots,C\}$ is associated with a learnable hash center $h_c\!\in\!\{-1,1\}^q$. 
For a sample $x_i$, the cosine similarity between its code and each class center is
\begin{equation}
\cos(u_i^c,h_c)=\frac{(u_i^c)^\top h_c}{\|u_i^c\|\|h_c\|}.
\end{equation}

The class probability is computed as
\begin{equation}
P_{i,c}=\frac{\exp(\sqrt{q}\cos(u_i^c,h_c))}{\sum_{r=1}^C \exp(\sqrt{q}\cos(u_i^c,h_r))}.
\end{equation}

The center loss encourages intra-class compactness and inter-class separability:
{\small
\begin{equation}\label{eq:LC}
L_C=-\frac{1}{N}\sum_{i=1}^N\sum_{c=1}^C
\big[y_{i,c}\log P_{i,c}+(1-y_{i,c})\log(1-P_{i,c})\big].
\end{equation}
}

\noindent\textbf{Pairwise Supervision.}
To preserve local similarity, we define a semantic similarity matrix
\begin{equation}
S_{ij}=\mathbb{I}(y_i^\top y_j>0),
\end{equation}
where $S_{ij}=1$ if $x_i$ and $x_j$ share at least one label and $0$ otherwise. 
Let $I_{ij}=\frac{1}{2}(u_i^p)^\top u_j^p$ denote the inner-product similarity of hash codes.
The pairwise loss is given by
{\small
\begin{equation}
L_P=\frac{1}{N}\sum_{i,j}\big[\log(1+e^{-|I_{ij}|})+\max(0,I_{ij})-S_{ij}I_{ij}\big].
\end{equation}
}

This objective captures fine-grained neighborhood relations and enhances generalization to unseen categories.

\noindent\textbf{Cross-branch Mutual Learning.}
To exchange information between branches, we introduce a cosine-based mutual loss:
\begin{equation}
L_M=\mathbb{E}[1-\cos(u^c,u^p)],
\end{equation}
where one branch is detached alternately per epoch to enable bidirectional knowledge flow.


\noindent\textbf{Default Weights.} Unless otherwise specified, we set $\lambda_1{=}4,\ \lambda_2{=}1,\ \lambda_3{=}1$. Impact of hyperparameter tuning is analyzed in Sec.~\ref{sec:ablation}.

\subsubsection{SM-MoH Module}
\label{subsec:MoH}

SM-MoH is a hashing-oriented adaptation of the classic MoE architecture, designed to split learning into branch-specific expert activations and merge them through a shared expert pool for coordinated semantic alignment.

\noindent\textbf{Split Gating and Merged Experts.}
Two independent gating networks, center-guided gating $G_c$ and pairwise-guided gating $G_p: \mathbb{R}^d \to \Delta^{m-1}$,  
produce expert weights 
for shared features $v_n$.
To encourage sparse and specialized routing, we select the top-$k$ experts with the highest activation scores:
\begin{equation}
\mathcal{K}_s(v_n) = \operatorname{TopK}(G_s(v_n), k),
\label{eq:topk_select}
\end{equation}
where $s \in \{c, p\}$.

The selected scores are then normalized to form a sparse routing distribution $\tilde{G}_s(v_n)$, ensuring weights of chosen experts sum to one. A set of shared experts $\{E_i\}_{i=1}^m$, each implemented as a lightweight neural network, then projects the input directly into the continuous hash code space $\mathbb{R}^q$:
\begin{equation}
E_i: \mathbb{R}^d \rightarrow \mathbb{R}^q.
\label{eq:expert_proj}
\end{equation}

For each branch $s \in \{c, p\}$, the aggregated representation is obtained by merging the expert responses under its gating distribution:
\begin{equation}
u_n^{s} = \sum_{i \in \mathcal{K}_s(v_n)} \tilde{G}_s(v_n)_i \cdot E_i(v_n),
\label{eq:merge_topk}
\end{equation}

\noindent\textbf{Design Comparison.} Unlike conventional MoE architectures that separate feature transformation and task-specific heads in Fig.~\ref{fig:MoH}(a), 
SM-MoH integrates expert transformation directly into the hashing process, as illustrated in Fig.~\ref{fig:MoH}(b)--(c). 
Each expert thus acts as an independent hashing pathway, producing semantically meaningful codes that reflect both branch-specific and shared characteristics. 
The final binary hash codes are computed as $b_n^{s} = \operatorname{sign}(u_n^{s})$.

We select the better-performing branch by mAP and use its binary codes for both queries and the database to perform Hamming-distance retrieval, where $s^*=\arg\max_{s\in\{c,p\}}\mathrm{mAP}(s)$ and $b^{s^*}=\mathrm{sign}\!\left(u^{s^*}\right)$.
This split--merge mechanism serves as a bridge for semantic alignment, enhancing the discriminability of the resulting hash codes.

\begin{algorithm}[t]
\caption{UniHash}
\begin{algorithmic}[1]
\Require Training set $\mathcal{D}$, number of experts $m$, top-$k$ value $k$, hash length $q$, iterations $T$, weights $\lambda_1, \lambda_2, \lambda_3$
\State \textbf{Initialize:} DNN backbone $\phi(\cdot)$, shared experts $\{E_i\}_{i=1}^m$, gating networks $G_{\text{c}}, G_{\text{p}}$, hash centers $\{\mathbf{h}_i\}_{i=1}^c$
\For{$t = 1$ to $T$}
    \State Sample mini-batch $X = \{x_1, \dots, x_N\}$ from $\mathcal{D}$
    \State Extract features $V_0 = \{\mathbf{v}_1, \dots, \mathbf{v}_N\} \gets \phi(X)$

    \For{each stream $s \in \{\text{c}, \text{p}\}$}
        \For{each sample $\mathbf{v}_n$ in $V_0$}
            \State Compute expert outputs: $\{E_i(\mathbf{v}_n)\}_{i=1}^m$
            \State Select top-$k$ experts: $\mathcal{K}_s(v_n) \gets \operatorname{TopK}(G_s(v_n), k)$
            \State Normalize selected weights: $\tilde{G}_s(v_n)_i \leftarrow \frac{G_s(v_n)_i}{\sum_{j \in \mathcal{K}_s(v_n)} G_s(v_n)_j}$ for $i \in \mathcal{K}_s(v_n)$
            \State Compute refined hash: $\mathbf{u}_n^{s} \gets \sum_{i \in \mathcal{K}_s(v_n)} \tilde{G}_s(v_n)_i \cdot E_i(v_n)$
        \EndFor
    \EndFor

    \State $L_C \leftarrow \text{CenterLoss}(\{\mathbf{u}_n^{\text{c}}\}, \{\mathbf{h}_i\})$
    \State $L_P \leftarrow \text{PairwiseLoss}(\{\mathbf{u}_n^{\text{p}}\})$

    \State \textbf{Compute mutual loss:}
    \If{$t \bmod 2 = 0$}
        \State Detach $\mathbf{u}_n^{\text{p}}$ as target
        \State $L_M \gets \frac{1}{N} \sum_{n=1}^N \big[1 - \cos(\mathbf{u}_n^{\text{c}}, \text{detach}(\mathbf{u}_n^{\text{p}}))\big]$
    \Else
        \State Detach $\mathbf{u}_n^{\text{c}}$ as target
        \State $L_M \gets \frac{1}{N} \sum_{n=1}^N \big[1 - \cos(\text{detach}(\mathbf{u}_n^{\text{c}}), \mathbf{u}_n^{\text{p}})\big]$
    \EndIf

    \State Total loss: $L \gets \lambda_1 L_C + \lambda_2 L_P + \lambda_3 L_M$
    \State Update $\phi$, $\{E_i\}$, $G_{\text{c}}, G_{\text{p}}$ via RMSProp using $L$
\EndFor
\State \Return Trained model: $\phi$, $\{E_i\}$, $G_{\text{c}}, G_{\text{p}}$
\end{algorithmic}
\end{algorithm}
\vspace{-3mm}

\subsection{Discrepancy Elimination Theory}
\label{sec:theory}

This section explains why UniHash yields stronger performance. Any single branch hashing paradigm, center-based or pairwise, suffers from a paradigm-specific structure discrepancy: an error floor that cannot be removed by increasing the sample size $n$ or the code length $q$. By transferring knowledge between the two branches, UniHash converts this fixed error into a vanishing consistency term, thereby asymptotically eliminating the structural discrepancy.

\subsubsection{Preliminaries: Assumptions and Capacity Terms}

We first establish the assumptions and fundamental capacity bounds that form the basis of our theoretical analysis.

\begin{assumption}[Boundedness, Sparsity \& Lipschitz]\label{bsl}
The feature extractor $\phi(x)$ is bounded: $\|\phi(x)\|\le B_\phi$.
Each expert $E_i$ is $\rho$-Lipschitz and satisfies $\|E_i(z)\|\le B_E$ for $\|z\|\le B_\phi$.
The gate activates at most $k$ experts per input. Under sparse expert routing, the statistical complexity contributes
\begin{equation}
\label{ass:complexity}
C_{stat} = O\!\left(\sqrt{\tfrac{k\,\log m}{n}}\right).
\end{equation}
\end{assumption}

\begin{assumption}[Center Separation]
\label{ass:center}
Pre-generated hash centers 
$\{\mathbf{h}_y \in \{-1,1\}^q\}_{y=1}^{c}$ 
satisfy a minimum Hamming distance 
$\mathrm{Ham}(\mathbf{h}_y,\mathbf{h}_{y'}) \ge d$ 
for $y \neq y'$.
\end{assumption}

\begin{assumption}[Semantic Manifolds]
\label{ass:pair}
The data reside on a union of semantic manifolds 
$\mathcal{M} = \bigcup_t \mathcal{M}_t$ 
with within-class diameter $\delta$ 
and inter-class separation $\Delta \gg \delta$ 
under a semantic metric.
Under Assumption~\ref{ass:center} and \ref{ass:pair}, quantization error decays exponentially in $q$:
\begin{equation}
\label{ass:quant}
C_{quant} = \exp(-\Omega(q)).
\end{equation}
\end{assumption}

\begin{assumption}[Mutual Consistency]
Mutual learning enforces cross-branch agreement:
\begin{equation*}
\mathbb{E}_{x}\!\left[
\|\mathbf{u}^{\mathrm{c}}(x)
-\mathbf{u}^{\mathrm{p}}(x)\|^2
\right]
\le \tau^2.
\end{equation*}
which contracts the joint hypothesis space and induces a consistency penalty.
\begin{equation}
C_{cons} = O(\tau/\sqrt{n}).
\end{equation}
\end{assumption}

\subsubsection{UniHash Superiority Analysis}
We investigate the superiority of UniHash by first characterizing the structural limitation of single-branch paradigms and then proving that UniHash removes this discrepancy through mutual consistency learning.

\begin{proposition}[Irreducible Structural Discrepancy]\label{prop:baseline-gap}
For a single-branch hashing method, the population risk obeys
\begin{equation}
\label{base}
R_{Base}
\;\le\;
\hat{R}_{Base} \;+\; C_{stat} \;+\; C_{quant} \;+\; E_{struct},
\end{equation}
where \(E_{struct}>0\) denotes a paradigm-specific, irreducible discrepancy independent of both $n$ and $q$. 
Specifically, a center-based method, which lacks relational constraints, incurs a \emph{relative} discrepancy \(E_{\mathrm{rel}}>0\), 
whereas a pairwise method, lacking global semantic anchors, suffers from an \emph{absolute} discrepancy \(E_{\mathrm{abs}}>0\).
\end{proposition}

\begin{theorem}[Discrepancy Elimination via UniHash]
Under all assumptions, the UniHash population risk satisfies
\begin{equation}
\label{UniHash}
R_{UniHash} \le \hat{R}_{UniHash} + C_{stat} + C_{cons} + C_{quant},
\end{equation}
where $C_{stat}$, $C_{cons}$, and $C_{quant}$ all vanish as $n, q \to \infty$.
\end{theorem}

\begin{proof}
By Proposition~\ref{prop:baseline-gap}, any single-branch model has a paradigm-induced error floor $E_{struct}>0$. 
The coupling loss $L_M$ in UniHash enforces branch agreement, replacing this floor by a consistency term 
$C_{cons}=\mathcal{O}(\tau/\sqrt{n})$. 
Together with the statistical and quantization terms, the UniHash risk satisfies
\begingroup\small
\begin{equation}
\small
\label{complex}
R_{UniHash}
\le
\hat{R}_{UniHash}
+\mathcal{O}\!\Big(\sqrt{\tfrac{k\log m}{n}}\Big)
+\mathcal{O}\!\Big(\tfrac{\tau}{\sqrt{n}}\Big)
+\exp\!\big(-\Omega(q)\big).
\end{equation}
\endgroup

Hence $C_{stat},C_{cons}\!\to 0$ as $n\to\infty$ and $C_\mathrm{quant}\!\to 0$ as $q\to\infty$, implying 
$R_{UniHash}-\hat{R}_{UniHash}\to 0$ and thus asymptotic elimination of the paradigm-specific discrepancy. \qedhere
\end{proof}

\section{Experiments}
\textbf{Datasets.} Following prior works \citep{yuan2020central,hoe2021one,wang2023deep,liu2016deep,li2015feature,cao2017hashnet,wang2017deep,su2018greedy,fan2020deep}, we evaluate performance on CIFAR-10 \citep{krizhevsky2009learning}, ImageNet \citep{deng2009imagenet}, and MSCOCO \citep{lin2014microsoft} for category-level retrieval. Evaluation metrics include mean average precision (mAP) and precision-recall curves. We report mAP@1000 for CIFAR-10 and ImageNet, and mAP@5000 for MSCOCO.

\textbf{Training Setup.} Following \citep{yuan2020central,hoe2021one,wang2023deep}, we use a pre-trained ResNet-50 \citep{he2016deep} as backbone $\phi(\cdot)$ and append a 4096-dimensional fully-connected ReLU layer \citep{agarap2019deep} after the 2048-dimensional pooled feature to generate the base features for hashing. These features are processed by a Split-Merge Mixture of Hash Experts (SM-MoH) module, consisting of $m$ shared expert networks $\{E_i\}_{i=1}^m$, two gating networks $G_{\text{c}}$ and $G_{\text{p}}$, and hash centers $\{\mathbf{h}_i\}_{i=1}^c$ for the center stream. Input images are resized to 224$\times$224, and we use a mini-batch size of $N=64$. The model is trained for $T=100$ epochs to optimize the backbone $\phi$, experts $\{E_i\}$, and gating networks $G_c$, $G_p$, by jointly optimizing $\lambda_1$, $\lambda_2$, and $\lambda_3$. The final binary hash code is obtained as $\text{sign}(\mathbf{u}_n^s)$, where $\mathbf{u}_n^s$ ($s \in \{{c}, {p}\}$) are refined continuous codes from center and pairwise streams. Our model is implemented in PyTorch and trained on NVIDIA RTX 4090 GPU using the RMSProp optimizer with learning rate of 0.0001.

\subsection{Closed-Set Retrieval Performance}

\begin{table*}[t]
    \centering
    \setlength{\tabcolsep}{4pt}
    \small
    \caption{Closed-set retrieval mAP on three datasets under different bit lengths. \textbf{Bold} indicates best performance; \underline{Underline} indicates second best performance. Pairwise methods are shaded. DAHNet$^{-}$ denotes DAHNet without class-label supervision.}
        \begin{tabular}{l ccc ccc ccc}
            \toprule
            \multirow{2}{*}{\textbf{Method}} & \multicolumn{3}{c}{\textbf{CIFAR-10(@1000)}} & \multicolumn{3}{c}{\textbf{ImageNet(@1000)}} & \multicolumn{3}{c}{\textbf{MSCOCO(@5000)}} \\
            \cmidrule(lr){2-4} \cmidrule(lr){5-7} \cmidrule(lr){8-10}
            & 16 bits & 32 bits & 64 bits & 16 bits & 32 bits & 64 bits & 16 bits & 32 bits & 64 bits \\
            \midrule
            \rowcolor{black!9}
            DSH~\citep{liu2016deep}  & 0.7313 & 0.7402 & 0.7272 & 0.7179 & 0.7448 & 0.7585 & 0.7221 & 0.7573 & 0.7790 \\
            \rowcolor{black!9}
            DPSH~\citep{li2015feature}  & 0.3098 & 0.3632 & 0.3638 & 0.6241 & 0.7626 & 0.7992 & 0.6239 & 0.6467 & 0.6322 \\
            \rowcolor{black!9}
            HashNet~\citep{cao2017hashnet}  & 0.8959 & 0.9115 & 0.8995 & 0.6024 & 0.7158 & 0.8071 & 0.7540 & 0.7331 & 0.7882 \\
            \rowcolor{black!9}
            DTSH~\citep{wang2017deep}  & 0.7783 & 0.7997 & 0.8312 & 0.6606 & 0.7803 & 0.8120 & \underline{0.7702} & 0.8105 & 0.8233 \\
            GreedyHash~\citep{su2018greedy}  & 0.3519 & 0.5350 & 0.6177 & 0.7394 & 0.7977 & 0.8243 & 0.7625 & 0.8033 & 0.8570 \\
            DPN~\citep{fan2020deep}  & 0.7576 & 0.7901 & 0.8040 & 0.7987 & 0.8298 & 0.8394 & 0.7571 & 0.8227 & 0.8623 \\
            CSQ~\citep{yuan2020central}  & 0.7861 & 0.7983 & 0.7989 & 0.8377 & 0.8750 & 0.8836 & 0.7509 & \underline{0.8471} & \underline{0.8610} \\
            OrthoHash~\citep{hoe2021one}  & 0.9087 & 0.9297 & 0.9454 & 0.8540 & 0.8792 & 0.8936 & 0.7174 & 0.7675 & 0.8060 \\            
            MDSH~\citep{wang2023deep}  & \underline{0.9455} & \underline{0.9554} & \underline{0.9607} & \underline{0.8639} & \underline{0.8863} & \underline{0.9019} & 0.7542 & 0.8131 & 0.8143 \\
            \rowcolor{black!9}
            AGMH~\citep{AGMHhashing} & 0.7831 & 0.7925 & 0.8148 & 0.8535 & 0.8568 & 0.8471 & 0.7031 & 0.6283 & 0.6577\\
            CFBH~\citep{CFBH} & 0.7187 & 0.7397 & 0.7580 & 0.8120 & 0.8572 & 0.8950 & 0.6807 & 0.8325 & 0.8521\\
            \rowcolor{black!9}
            DAHNet$^-$~\citep{DAHNet} & 0.8932 & 0.9087 & 0.9125 & 0.8325 & 0.8831 & 0.8921 & 0.7481 & 0.7993 & 0.8053\\
            \midrule
            \textbf{Ours}  & \textbf{0.9665} & \textbf{0.9657} & \textbf{0.9658} & \textbf{0.8744} & \textbf{0.8975} & \textbf{0.9062} & \textbf{0.7903} & \textbf{0.8675} & \textbf{0.8727} \\
            \bottomrule
        \end{tabular}
    \label{tab:sota_compare}
\end{table*}

\begin{table*}[t]
\centering
\small 
\setlength{\tabcolsep}{4pt} 
\caption{Retrieval under seen/unseen splits. mAP (\%) for Seen@Seen, Seen@All, Unseen@Unseen, and Unseen@All on 20\%-unseen CIFAR-10 and 15\%-unseen ImageNet/MSCOCO. \textbf{Bold}: best; \underline{Underlined}: second or third best. Pairwise methods are shaded. DAHNet$^{*}$ denotes the label-supervised variant of DAHNet.}
\label{tab:seen-unseen-oneheader}
\begin{tabular}{l ccc ccc ccc}
\toprule
& \multicolumn{3}{c}{\textbf{CIFAR-10(@1000)}} & \multicolumn{3}{c}{\textbf{ImageNet(@1000)}} & \multicolumn{3}{c}{\textbf{MSCOCO(@5000)}} \\
\cmidrule(lr){2-4}\cmidrule(lr){5-7}\cmidrule(lr){8-10}
\textbf{Method} & 16 bits & 32 bits & 64 bits & 16 bits & 32 bits & 64 bits & 16 bits & 32 bits & 64 bits \\
\midrule
\multicolumn{10}{l}{\textbf{(a) Seen@Seen / Seen@All}} \\
\addlinespace[2pt]
CSQ~\citep{yuan2020central}       &\uline{95.2}/\uline{95.1} & \uline{95.1/95.0} & 95.0/\uline{94.9} & 85.3/85.0 & 89.1/\uline{88.7} & \uline{90.6/90.2} & 65.0/\uline{79.7} & 68.8/\uline{81.9} & 70.1/\uline{82.4} \\
OrthoHash~\citep{hoe2021one} & 94.6/92.8 & \uline{95.1}/93.2 & \uline{95.7}/\uline{93.7} & \uline{88.2/87.7} & \uline{90.3/89.8} & \uline{91.5/90.9} & 62.7/78.7 & 65.8/80.5 & 68.5/82.4 \\
MDSH~\citep{wang2023deep}      & 92.8/92.3 & 92.5/92.0 & 92.8/89.9 & 75.5/75.3 & 78.8/78.5 & 81.9/81.1 & 64.2/\uline{80.1} & \uline{70.3/81.2} & \uline{70.9/83.3} \\
CHN~\citep{cao2017CHN}       & \uline{95.8/95.7} & \uline{95.9/95.5} & \uline{95.6}/90.1 & \uline{86.4/86.0} & 88.9/88.5 & 90.1/89.5 & -- & -- & -- \\
\rowcolor{black!9}
DSH~\citep{liu2016deep}      & 92.6/91.9 & 92.5/91.8 & 92.9/92.0 & 68.8/67.9 & 79.5/78.2 & 85.3/84.0 & 46.8/64.6 & 51.1/69.7 & 51.4/71.8 \\
\rowcolor{black!9}
HashNet~\citep{cao2017hashnet}   & 80.5/80.1 & 92.7/92.2 & 94.3/93.4 & 49.1/48.9 & 75.8/75.0 & 86.5/85.6 & 52.4/74.8 & 59.9/79.9 & 64.6/82.4 \\
\rowcolor{black!9}
ADSH~\citep{ADSH}   & 95.3/67.0 & 95.0/69.2 & 93.8/61.7 & 83.5/82.7 & \uline{89.7}/88.6 & 90.4/89.3 & \uline{75.1}/66.7 & \uline{79.7}/71.3 & \uline{80.0}/68.3 \\
\rowcolor{black!9}
AGMH~\citep{AGMHhashing}  & 81.0/41.5 & 79.4/12.5 & 81.0/5.5 & 85.8/85.2 & 87.4/86.0 & 84.0/82.6 & \uline{69.0}/58.2 & 59.7/38.8 & 64.1/39.7 \\
\rowcolor{black!9}
DAHNet$^\ast$~\citep{DAHNet}   & 94.8/67.5 & 93.7/63.2 & 92.2/58.6 & 86.1/85.3 & 88.8/87.5 & 88.6/87.5 & -- & -- & -- \\

\midrule
\textbf{Ours} & \textbf{97.3/97.0} & \textbf{97.2/97.1} & \textbf{97.2/97.0} & \textbf{89.9/88.3} & \textbf{92.2/89.9} & \textbf{92.4/91.5} & \textbf{79.3/81.0} & \textbf{81.5/82.3} & \textbf{85.2/83.9} \\
\midrule
\multicolumn{10}{l}{\textbf{(b) Unseen@Unseen / Unseen@All}} \\
\addlinespace[2pt]
CSQ~\citep{yuan2020central}       & 77.3/7.2  & 73.1/10.2 & 75.7/9.4  & 38.1/19.6 & 51.7/\uline{36.2} & 55.8/\uline{42.2} & 46.9/50.1 & \uline{51.2}/54.0 & 52.9/55.8 \\
OrthoHash~\citep{hoe2021one} & \uline{82.0}/31.6 & \uline{83.0}/36.0 & \uline{82.1}/36.3 & 41.3/\uline{25.1} & 50.7/34.8 & \uline{59.9}/40.5 & \uline{49.2}/\uline{54.7} & 50.9/\uline{55.9} & 53.3/\uline{58.1} \\
MDSH~\citep{wang2023deep}      & 76.3/18.8 & 79.4/35.0 & \uline{81.4}/\uline{45.2} & 28.3/10.0 & 31.8/14.8 & 41.1/21.1 & \uline{48.6}/\uline{55.9} & 50.5/\uline{56.2} & \uline{53.4}/57.4 \\
CHN~\citep{cao2017CHN}         & 73.2/6.6  & 73.2/11.6 & 78.0/39.0 & 37.0/18.0 & 46.9/30.8 & 56.2/39.0 & -- & -- & -- \\
\rowcolor{black!9}
DSH~\citep{liu2016deep}       & 74.5/11.0 & 70.8/18.3 & 72.4/20.0 & \uline{42.6}/20.9 & 49.6/26.8 & 56.1/34.6 & 41.1/45.0 & 44.1/47.1 & 46.3/50.2 \\
\rowcolor{black!9}
HashNet~\citep{cao2017hashnet}    & 78.8/9.4  & \uline{77.7}/24.7 & 80.8/38.2 & \uline{42.3}/18.9 & \uline{57.5/32.3} & \uline{71.2/44.8} & 44.5/45.9 & \uline{51.8}/55.8 & \uline{55.9/59.2} \\
\rowcolor{black!9}
ADSH~\citep{ADSH}   & \uline{79.0}/\uline{78.7} & 74.4/59.3 & 75.1/\uline{75.1} & \uline{42.3}/\uline{21.8} & \uline{55.4}/30.7 & 58.0/33.4 & 41.1/44.2 & 50.9/52.7 & 51.6/52.2 \\
\rowcolor{black!9}
AGMH~\citep{AGMHhashing}  & 77.0/\uline{77.2} & 75.4/\uline{75.4} & 73.2/73.2 & 36.4/16.1 & 44.2/21.1 & 48.9/25.6 & 41.1/42.5 & 37.5/37.5 & 40.0/40.0 \\
\rowcolor{black!9}
DAHNet$^\ast$~\citep{DAHNet}  & 76.2/73.9 & 77.0/\uline{76.9} & 75.8/\uline{75.1} & 37.2/16.7 & 46.2/20.5 & 59.6/34.8 & -- & -- & -- \\
\midrule
\textbf{Ours} & \textbf{84.7/82.3} & \textbf{85.2/85.1} & \textbf{85.6/83.8} & \textbf{44.7/28.0} & \textbf{59.1/38.0} & \textbf{71.3/48.2} & \textbf{53.2/58.7} & \textbf{57.9/60.3} & \textbf{58.4/62.1} \\
\bottomrule
\end{tabular}
\end{table*}
\vspace{-1mm}

We compare our method with nine deep hashing algorithms: five pointwise methods (DPN \citep{fan2020deep}, GreedyHash \citep{su2018greedy}, CSQ \citep{yuan2020central}, OrthoHash \citep{hoe2021one}, MDSH \citep{wang2023deep}), three pairwise methods (DSH \citep{liu2016deep}, DPSH \citep{li2015feature}, HashNet \citep{cao2017hashnet}), AGMH~\citep{AGMHhashing},  DAHNet$^-$~\citep{DAHNet}, and one tripletwise method (DTSH \citep{wang2017deep}). Table~\ref{tab:sota_compare} reports the Mean Average Precision (mAP) results for image retrieval. We adopt ResNet-50 as the backbone for all compared methods, including pairwise and pointwise baselines, and our proposed UniHash. Compared to these methods, our method achieves mAP improvements of 1.74\%, 1.21\%, and 0.87\% on MSCOCO, CIFAR-10, and ImageNet, respectively, averaged across different code lengths. 

We further evaluate retrieval performance using Precision-Recall (PR) curves, as shown in Fig.~\ref{fig:pr_curve}. Our method consistently yields a larger Area Under the PR Curve (AUC-PR) across all bit lengths, demonstrating superior precision across a wide range of recall values. These results underscore the robustness and generalization capability of our method for large-scale image retrieval.

\subsection{Generalization to Unseen Categories}
\vspace{-1mm}

To evaluate retrieval generalization on unseen classes, following~\cite{shen2025empirical}, we split datasets into seen and unseen categories, holding out 20\% of CIFAR-10 and 15\% of ImageNet and MSCOCO as unseen and excluding them from training. Table~\ref{tab:seen-unseen-oneheader} reports mAP for four cases: (a) Seen@Seen and Seen@All: queries from seen classes searching among databases of seen-only or all images, and (b) Unseen@Unseen and Unseen@All: queries from novel classes searching among only unseen images or the entire database. It is worth noting that this constitutes a more challenging zero-shot hashing task~\cite{shen2025empirical}. Although zero-shot hashing has been studied in prior works~\cite{10197227,cao2017hashnet}, their experimental settings typically designate only a single category as unseen during testing, which provides a limited assessment of a method’s retrieval capability when facing multiple unseen categories.

As shown in Table~\ref{tab:seen-unseen-oneheader}, UniHash demonstrates robust performance on both seen and unseen retrieval tasks. For seen-category queries, UniHash achieves near-perfect accuracy (\emph{e.g.}, ~97\% mAP on CIFAR-10) and sustains this on the Seen@All task. Competing methods also perform well but often drop slightly on Seen@All due to interference from unseen classes. More importantly, UniHash significantly outperforms prior methods on unseen-category queries. On CIFAR-10, UniHash achieves 84.7\% mAP on Unseen@Unseen (vs. 82.0\% for OrthoHash) and maintains a much higher mAP on Unseen@All (82.3\% vs. 7.2\% for CSQ). Similar trends hold on ImageNet and MSCOCO, where UniHash yields the highest Unseen@Unseen accuracy and clearly better Unseen@All mAP (with about 4\% absolute gains over the best baselines). These results indicate that UniHash enables strong generalization, retrieving novel-class images even with many distractors.
\vspace{-1mm}
\subsection{Ablation Study} 
\label{sec:ablation}

We conduct an ablation study on three datasets under different hash code lengths (16, 32, and 64 bits) to evaluate the effectiveness of the proposed components in our deep hashing network, including both the overall framework modules in Table~\ref{tab:ablation} and the detailed design choices of the Split-Merge Mixture of Hash Experts (SM-MoH) head in Table~\ref{tab:model_comparison}.

\textbf{Analysis of Overall Framework.} To evaluate the contribution of each component in overall architecture, we conducted a closed-set ablation study on both the SM-MoH and Mutual Learning (ML) loss modules, as shown in Table~\ref{tab:ablation}. Baseline (without either SM-MoH or ML) shows that the center-based branch consistently outperforms the pairwise branch (\emph{e.g.}, 0.8940 vs. 0.8873 on ImageNet at 64 bits), indicating stronger standalone effectiveness. ML alone improves the weaker pairwise branch (\emph{e.g.}, 0.8894 vs. 0.8873 pairwise branch on ImageNet at 64 bits), by enabling bidirectional knowledge transfer. Gains for branches are smaller due to limited diversity without SM-MoH. SM-MoH is not a stand-alone architectural add-on; it is designed to facilitate cross-branch structured interaction. In isolation, SM-MoH slightly reduces performance (\emph{e.g.}, center-based: 0.8325 vs.\ 0.8940), whereas combining SM-MoH with mutual learning yields the best results (\emph{e.g.}, 0.8997/0.9062 on ImageNet at 64 bits). This trend suggests that SM-MoH strengthens inter-branch communication, and mutual learning further amplifies complementary supervision, jointly improving intra-class compactness and inter-class separability.

\begin{figure}[t]
\centering
    \begin{subfigure}[b]{0.33\linewidth}
        \centering
        \includegraphics[width=\linewidth]{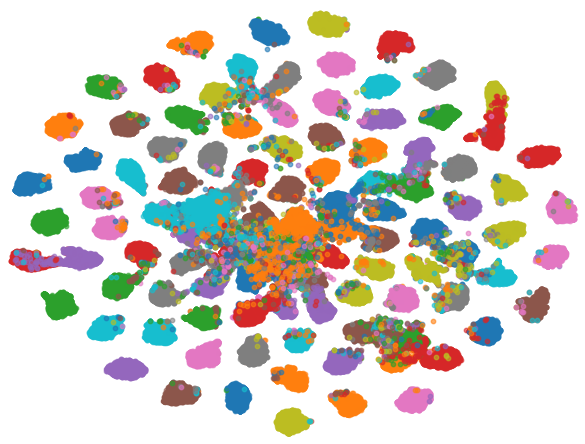}
        \caption{Center Hashing}
    \end{subfigure}\hfill
    \begin{subfigure}[b]{0.33\linewidth}
        \centering
        \includegraphics[width=\linewidth]{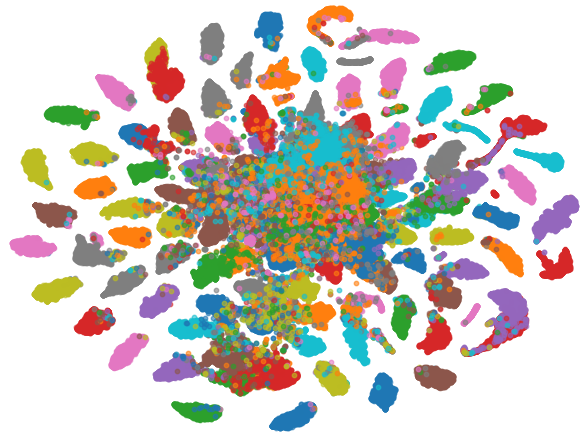}
        \caption{Pairwise Hashing}
    \end{subfigure}\hfill
    \begin{subfigure}[b]{0.33\linewidth}
        \centering
        \includegraphics[width=\linewidth]{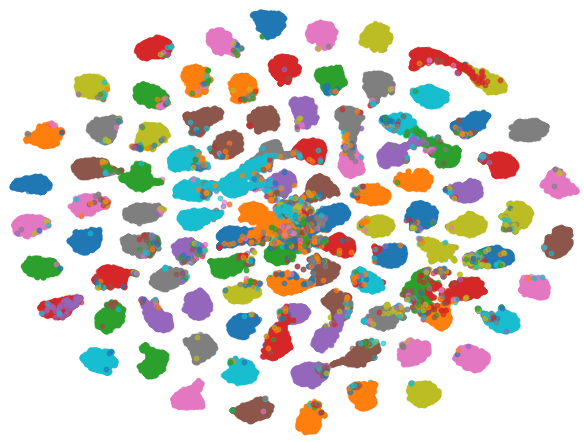}
        \caption{Unified Hashing}
    \end{subfigure}
\caption{Impact of Unified Hashing (UniHash) on hash code distribution on ImageNet with 16-bit configuration.}
\label{fig:comparison}
\end{figure}

\begin{figure}[t]
    \centering
    \begin{subfigure}[t]{0.33\linewidth}
        \centering
        \includegraphics[width=\linewidth]{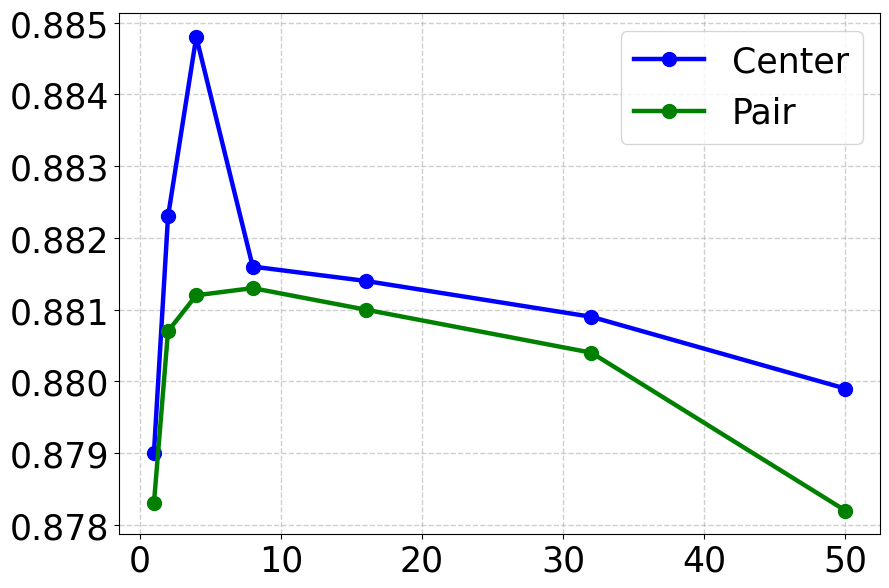}
        \caption{$\lambda_1$ Tuning Effects}
    \end{subfigure}\hfill
    \begin{subfigure}[t]{0.33\linewidth}
        \centering
        \includegraphics[width=\linewidth]{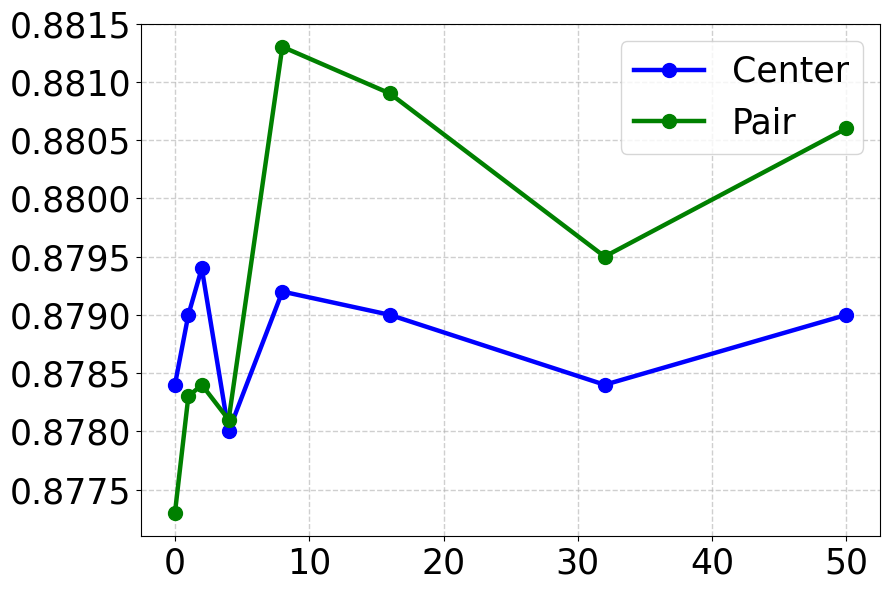}
        \caption{$\lambda_2$ Tuning Effects}
    \end{subfigure}\hfill
    \begin{subfigure}[t]{0.33\linewidth}
        \centering
        \includegraphics[width=\linewidth]{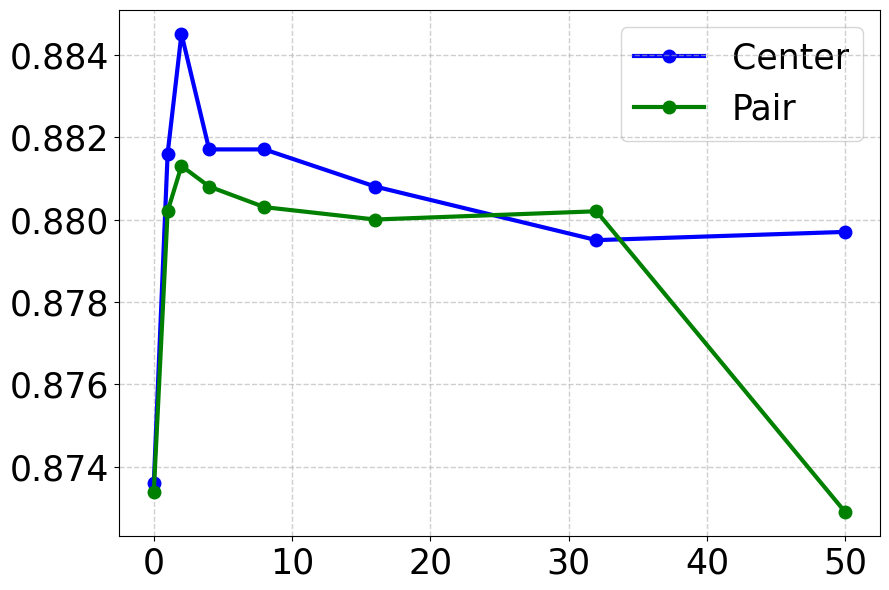}
        \caption{$\lambda_3$ Tuning Effects}
    \end{subfigure}
    \caption{Impact of hyperparameter tuning on mAP for ImageNet with 16-bit configuration.}
    \label{fig:hyperparameters}
\end{figure}


\begin{table}[t]
\centering
\caption{Ablation study of Mutual learning (ML) and SM-MoH across different datasets for 64 bits. The best results are \textbf{bolded}.}
\label{tab:ablation}
\resizebox{1.0\linewidth}{!}{ 
\setlength{\tabcolsep}{2pt}
\footnotesize 
\begin{tabular}{ccc *{4}{c}}
\toprule
 & \multicolumn{2}{c}{Modules} & \multicolumn{2}{c}{\textbf{CIFAR-10(@1000)}} & \multicolumn{2}{c}{\textbf{ImageNet(@1000)}} \\
\cmidrule(lr){2-3} \cmidrule(lr){4-5} \cmidrule(lr){6-7}
Baseline & ML & SM-MoH & center-based & pairwise & center-based & pairwise \\
\midrule
$\checkmark$ &  &  & 0.9607 & 0.9605 & 0.8940 & 0.8873 \\
$\checkmark$ & $\checkmark$ &  & 0.9586 & 0.9634 & \textbf{0.9037} & 0.8894 \\
$\checkmark$ &  & $\checkmark$ & 0.9357 & 0.9639 & 0.8325 & 0.9003 \\
$\checkmark$ & $\checkmark$ & $\checkmark$ & \textbf{0.9611} & \textbf{0.9658} & 0.8997 & \textbf{0.9062} \\
\bottomrule
\end{tabular}
}
\end{table}


\begin{table}[t]
\centering
\caption{Ablation study of SM-MoH design variants with 64-bit codes. $\sigma(\cdot)$ is softmax function. Best results are in \textbf{bold}.}
\label{tab:model_comparison}
\small
\setlength{\tabcolsep}{1.5pt} 
\renewcommand{\arraystretch}{1.0} 
\begin{tabular}{llcccc}
\toprule
Model & Experts & $\sigma(\cdot)$ & \textbf{CIFAR-10} & \textbf{ImageNet} & \textbf{MSCOCO} \\
\midrule
2DNNs+MoE & separate & $\checkmark$ & 0.9634 & 0.8862 & 0.8605 \\
1DNN+MoE  & separate & $\checkmark$ & 0.9643 & 0.9003 & 0.8693 \\
1DNN+MoH  & separate & $\times$     & 0.9651 & 0.9033 & 0.8697 \\
1DNN+MoH  & shared   & $\checkmark$ & 0.9647 & 0.9012 & 0.8684 \\
1DNN+MoH  & shared   & $\times$     & \textbf{0.9658} & \textbf{0.9062} & \textbf{0.8727} \\
\bottomrule
\end{tabular}
\end{table}
\vspace{-2mm}





\textbf{Design Variation in SM-MoH.} We next evaluate several design alternatives for the SM-MoH in Table~\ref{tab:model_comparison}. The conventional setup (2DNNs+MoE) employs two separate DNNs with expert modules. It achieves noticeable improvements over the baseline (\emph{e.g.}, 0.9643 vs. 0.9607 on CIFAR-10), demonstrating the benefit of collaborative learning. However, to further enhance performance and simplify the architecture, we explore alternative designs. Switching to a single-branch structure (1DNN+MoE) leads to slightly improved results (\emph{e.g.}, 0.9634 vs. 0.9643 on CIFAR-10; 0.8605 vs. 0.8693 on MSCOCO), confirming that a unified backbone provides more coherent feature learning for hashing tasks. Replacing MoE with our proposed SM-MoH module—which maps features directly to the hash space—further improves performance (\emph{e.g.}, 0.9651 on CIFAR-10, 0.9033 on ImageNet), demonstrating its better alignment with deep hashing objectives. This gain holds for both separate and shared expert configurations. Sharing experts across branches removes redundancy and enhances learning consistency, with shared MoH slightly outperforming its separate counterpart (\emph{e.g.}, 0.9647 vs. 0.9651 on CIFAR-10; 0.8684 vs. 0.8697 on MSCOCO). Removing the softmax layer yields the best performance across all datasets (\emph{e.g.}, 0.9658 on CIFAR-10, 0.9062 on ImageNet, 0.8727 on MSCOCO), likely due to improved code separability by avoiding over-smoothing among experts.

\textbf{Hyperparameter Tuning.} Figure~\ref{fig:hyperparameters} shows a hyperparameter study on ImageNet with 16-bit hash codes, varying one of $\lambda_1$, $\lambda_2$, or $\lambda_3$ while fixing the others. $\lambda_1$ and $\lambda_2$ control the center- and pairwise-based losses, and $\lambda_3$ regulates mutual loss. The results reveal a dominant–auxiliary dynamic: when $\lambda_{1} \gg \lambda_{2}$, the center branch outperforms; conversely, when $\lambda_{2} \gg \lambda_{1}$, the pairwise branch outperforms. Configurations with $\lambda_1 > \lambda_2$ achieve higher mAP, so the center branch serves as the primary while the pairwise branch fine-tunes it. Both $\lambda_1$ and $\lambda_2$ exhibit unimodal trends, with optimal values within $[1,10]$. And $\lambda_3$ performs best at $1$, beyond which excessive coupling degrades performance. The optimal setting is $\lambda_1=4$, $\lambda_2=1$, $\lambda_3=1$.

\subsection{Hash Codes Visualization}
To illustrate UniHash, we visualize the t-SNE \citep{vanDerMaaten2008TSNE} of hash codes generated by three methods on ImageNet, as shown in Fig.~\ref{fig:comparison}: (a) center-based \citep{wang2023deep}, (b) pairwise-based \citep{zheng2020deep}, and (c) our UniHash approach. The center-based method in (a) produces a circular distribution, indicating decent global alignment but limited fine-grained separation. The pairwise-based method in (b) yields a more elongated structure with overlapping clusters, suggesting weaker overall structure. In contrast, UniHash in (c) consistently leads to a more well-clustered distribution, with fewer ambiguous points across class boundaries. These results show that UniHash improves intra-class compactness and inter-class separability by making the center- and pairwise-based branches complementary, yielding more discriminative hash codes.
\vspace{-2mm}

\section{Conclusions}


We propose Unified Hashing (UniHash), a dual-branch framework designed to achieve good image retrieval performance across both seen and unseen categories. UniHash comprises two complementary branches: one following the pointwise training paradigm and the other following the pairwise paradigm. To strengthen cross-branch interaction and representation alignment, we introduce a novel hash code learning method that facilitates deeper information exchange between branches, yielding more discriminative and generalizable hash codes. Theoretical analysis verifies the effectiveness of UniHash, and extensive experiments on CIFAR-10, MSCOCO, and ImageNet demonstrate state-of-the-art performance under both closed-set and seen/unseen retrieval protocols, confirming its robustness across diverse retrieval settings.

\clearpage

\bibliographystyle{unsrtnat}
\bibliography{references.bib}

\clearpage
\appendix
\onecolumn

\appendix

\section{Summary}

This appendix provides detailed insights and additional experimental results to support the main paper.

\section{Closed-Set Training Performance}

\begin{figure}[htbp]
    \centering
    \begin{subfigure}[t]{0.33\linewidth}
        \centering
        \includegraphics[width=\linewidth]{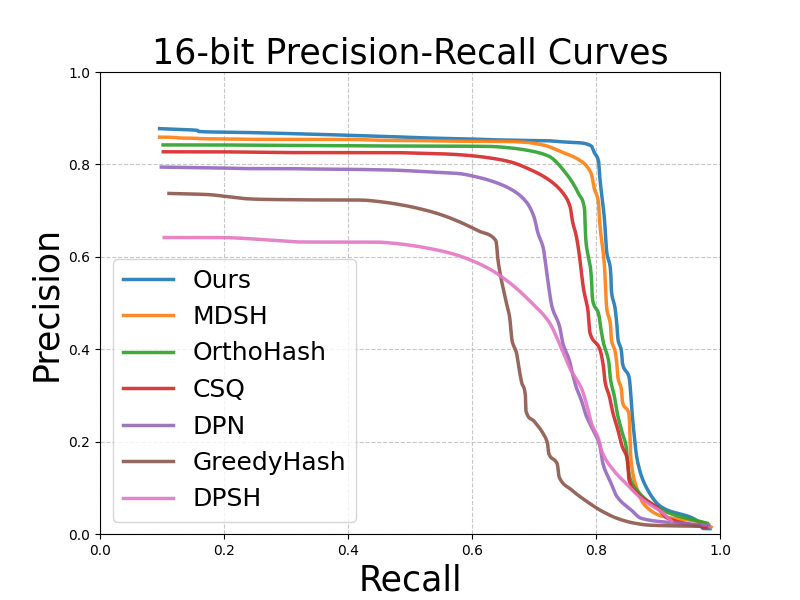}
        \caption{16 bit PR Curves}
    \end{subfigure}\hfill
    \begin{subfigure}[t]{0.33\linewidth}
        \centering
        \includegraphics[width=\linewidth]{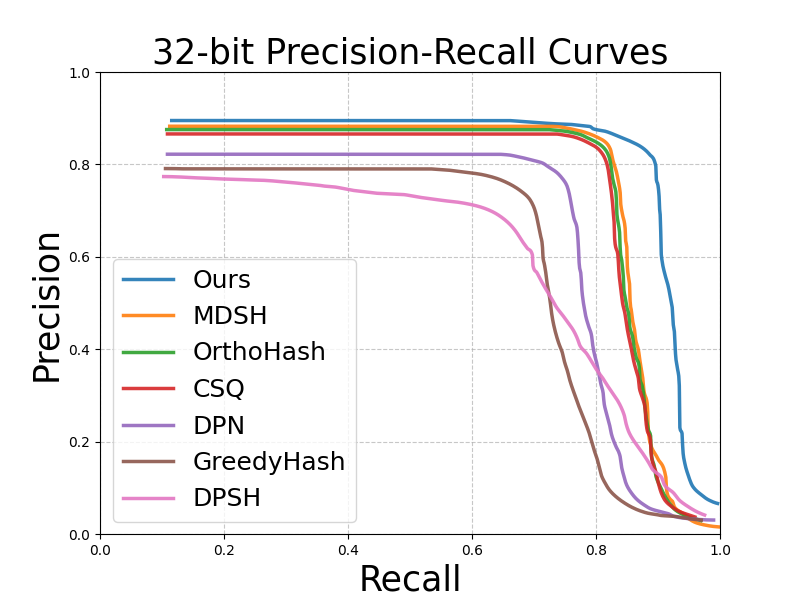}
        \caption{32 bit PR Curves}
    \end{subfigure}\hfill
    \begin{subfigure}[t]{0.33\linewidth}
        \centering
        \includegraphics[width=\linewidth]{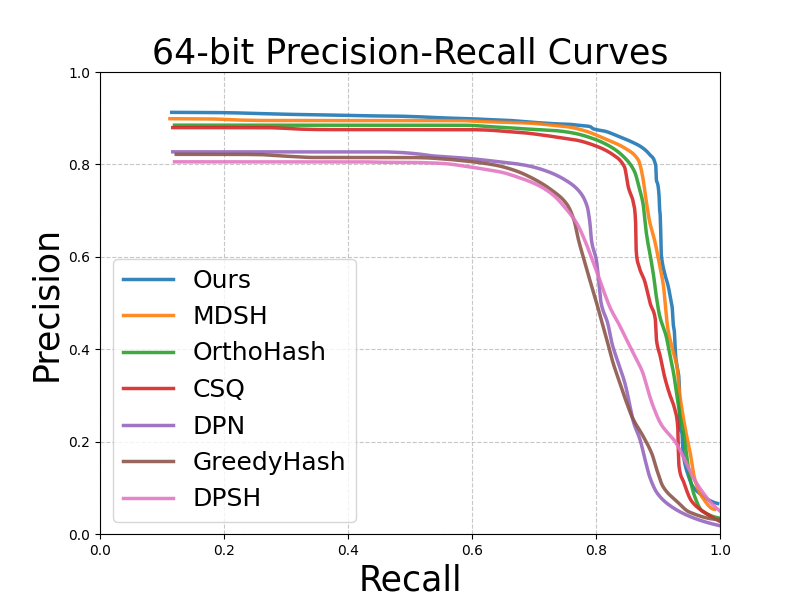}
        \caption{64 bit PR Curves}
    \end{subfigure}
    \caption{Precision-Recall curves on ImageNet across different bit configurations.}
    \label{fig:pr_curve}
\end{figure}
\vspace{-2mm}

\section{Training Setup}

\subsection{Datasets}

\textbf{ImageNet} is a large-scale image classification dataset consisting of over 1.2 million images annotated with 1,000 categories. Following the protocol in~\citep{deng2009imagenet}, we use the ILSVRC2012 version for evaluation. The validation set of 50K images is used as the query set, while the remaining training images form the database. For training, we randomly sample 130K images from the database.

\textbf{CIFAR-10} consists of 60,000 images across 10 categories, with each image sized $32\times32$. Following the standard practice in~\citep{krizhevsky2009learning}, we use 10K test images as the query set and the remaining 50K training images as the database. For training, 5K images are randomly sampled from the database.

\textbf{MSCOCO} ~\citep{lin2014microsoft} is an image recognition, segmentation, and captioning dataset. We use the public version processed by~\citep{lin2014microsoft}, where images with missing category information have been filtered out. This results in 122K labeled images by combining the training and validation splits. We randomly sample 5K images as the query set, with the remaining images forming the database, and then randomly sample 10K images from the database for training.

\textbf{Seen/Unseen Splitting}~\citep{shen2025empirical} refers to our unified protocol for constructing seen and unseen subsets across both single-label and multi-label datasets. For single-label datasets such as CIFAR-10 and ImageNet-1K, we partition the label space using dataset-specific ratios: in CIFAR-10, 80\% of the categories are designated as seen, whereas in ImageNet, 85\% of the categories are treated as seen; the remaining categories in each dataset serve as unseen. For the multi-label dataset MSCOCO, we likewise divide the label space by a category ratio (85\% seen). However, since images may contain multiple tags, an image is included in the seen or unseen subset only when all of its labels fall entirely within the corresponding class group, while images with mixed labels are used exclusively in the “all-class” database settings. Using these purified subsets, we then construct the four evaluation protocols—Seen@Seen, Seen@All, Unseen@Unseen, and Unseen@All—in a consistent manner across all datasets.

\paragraph{License.} 
ImageNet is released under a non-commercial license, and the use of the dataset is restricted to research and educational purposes. Users must apply for access and agree to the ImageNet Terms of Use. CIFAR-10 is made publicly available by the University of Toronto under the MIT License. This permits free use, modification, and distribution of the dataset for both research and commercial purposes. For MSCOCO, the annotations are provided under the Creative Commons Attribution 4.0 License (CC BY 4.0), and the use of the images must comply with the Flickr Terms of Use. The dataset is released for academic and research use.

\section{Ablation Study and Further Analysis}
\subsection{Comparison with Traditional Mutual Learning}

Traditional Deep Mutual Learning (DML) \citep{zhang2018deep,wu2019mutual,zhao2021deep,zhao2021novel} typically employs two separate branches, where each branch contains an independently initialized and trained deep neural network (DNN). While this design allows for mutual supervision between diverse learners, it also limits the potential for fine-grained interaction between the learned representations, especially in the context of hashing where compact and consistent binary codes are desired.

In contrast, our method adopts a shared-backbone design with two branches operating on the same DNN. This encourages closer interaction and more effective information sharing between the branches, thereby facilitating the generation of more consistent and semantically aligned hash codes. The underlying idea is to enforce mutual guidance without introducing significant representational discrepancies caused by separate networks.

We evaluate both settings — one with a single shared DNN (denoted as 1DNN+MoH), and one with two independent DNNs (denoted as 2DNN+MoH) — across three benchmark datasets: ImageNet, MSCOCO, and CIFAR-10. As shown in Table~\ref{tab:dml_comparison}, our shared DNN design consistently outperforms the traditional dual-DNN setup across almost all bit lengths and datasets.

\begin{table*}[htbp]
\small
\centering
\caption{Performance of 1DNN+MoH and 2DNN+MoH on ImageNet, MSCOCO, and CIFAR-10 across all bits. Best results are \textbf{bolded}.}
\begin{tabular}{l*{9}{c}}
\toprule
 & \multicolumn{3}{c}{\textbf{ImageNet}} & \multicolumn{3}{c}{\textbf{MSCOCO}} & \multicolumn{3}{c}{\textbf{CIFAR-10}} \\
\cmidrule(lr){2-4} \cmidrule(lr){5-7} \cmidrule(lr){8-10}
\textbf{Method} & \textbf{16 bits} & \textbf{32 bits} & \textbf{64 bits} & \textbf{16 bits} & \textbf{32 bits} & \textbf{64 bits} & \textbf{16 bits} & \textbf{32 bits} & \textbf{64 bits} \\
\midrule
2DNN+MoH & 0.8601 & 0.8859 & 0.8996 & 0.7374 & 0.8217 & 0.8623 & 0.9632 & 0.9650 & 0.9647 \\
\textbf{1DNN+MoH} & \textbf{0.8744} & \textbf{0.8975} & \textbf{0.9062} & \textbf{0.7903} & \textbf{0.8675} & \textbf{0.8727} & \textbf{0.9665} & \textbf{0.9657} & \textbf{0.9658} \\
\bottomrule
\label{tab:dml_comparison}
\end{tabular}
\end{table*}

\subsection{SM-MoH Module Analysis}

To better understand the efficacy of our proposed Split-Merge Mixture of Hash Experts (SM-MoH) module, we conduct ablation studies targeting three core components: the design of hashing experts, expert sharing, and the role of the softmax mechanism. Table~\ref{tab:moh_analysis} summarizes the experimental results across three datasets.

\textbf{Design of Hashing Experts vs. Traditional Experts.} Traditional Mixture of Experts (MoE) \citep{chen2023adamv,chen2022towards,riquelme2021scaling,shazeer2017outrageously} typically employs two-layer MLPs with ReLU activations as experts, designed for general-purpose representation transformation. In contrast, our MoH replaces these with specialized hashing experts, which directly map the feature dimension to the hash bit dimension — effectively acting as task-specific hashing layers.

We compare two baselines: the traditional MoE expert, which uses a two-layer MLP with hidden ReLU as commonly seen in the MoE literature, and the traditional hash expert, which consists of a single linear projection layer without non-linearity, as typically employed in hashing methods.

Our design strikes a balance: it retains the structure of two-layer MLPs but aligns their output directly to binary codes, offering greater representational power while preserving hash compatibility. As shown in the table, both traditional variants perform worse than our method, especially on MSCOCO (\emph{e.g.}, 0.8675 vs. 0.8472 for 32-bit).

\textbf{Expert Sharing Across Branches.} We adopt a shared expert design across branches in MoH to encourage consistent hashing and reduce redundancy. To verify its effectiveness, we compare with a variant where each branch has separate (unshared) experts.

Results show that unshared experts degrade performance on all datasets. For instance, on ImageNet at 32-bit, shared experts achieve 0.8975 vs. 0.8958 with unshared experts, confirming that expert sharing enhances generalization and code consistency.

\textbf{Impact of Removing Softmax Gate.} Unlike traditional MoE which utilizes softmax to weigh expert contributions, we remove softmax and instead allow parallel supervision from all experts. This simplifies optimization and encourages more diverse expert behaviors. As Table~\ref{tab:moh_analysis} shows, removing softmax leads to consistent improvements: for instance, on ImageNet (64-bit), performance rises from 0.9001 (with softmax) to 0.9062 (ours).

\begin{table*}[htbp]
\small
\centering
\caption{Performance comparison of different methods on ImageNet, MSCOCO, and CIFAR-10 across all bits. Traditional MoE experts typically employ two-layer MLPs with ReLU activations. Traditional hashing expert consists of a single linear projection layer without non-linearity, as commonly used in hashing-based methods. Unshared expert indicates that the two branches do not share experts. "Trad" in this table means traditional. Best results are \textbf{bolded}.}
\begin{tabular}{l*{9}{c}}
\toprule
 & \multicolumn{3}{c}{\textbf{ImageNet}} & \multicolumn{3}{c}{\textbf{MSCOCO}} & \multicolumn{3}{c}{\textbf{CIFAR-10}} \\
\cmidrule(lr){2-4} \cmidrule(lr){5-7} \cmidrule(lr){8-10}
\textbf{Method} & \textbf{16 bit} & \textbf{32 bit} & \textbf{64 bit} & \textbf{16 bit} & \textbf{32 bit} & \textbf{64 bit} & \textbf{16 bit} & \textbf{32 bit} & \textbf{64 bit} \\
\midrule
Trad MoE Expert & 0.8762 & 0.8862 & 0.8942 & 0.7730 & 0.8472 & 0.8605 & 0.9649 & 0.9653 & 0.9634 \\
Trad Hash Expert & 0.8663 & 0.8872 & 0.8996 & 0.7584 & 0.8533 & 0.8658 & 0.9638 & 0.9641 & 0.9649 \\
Unshared Expert & 0.8728 & 0.8958 & 0.9031 & 0.7882 & 0.8657 & 0.8701 & 0.9658 & 0.9659 & 0.9657 \\
With Softmax & 0.8679 & 0.8907 & 0.9001 & 0.7793 & 0.8526 & 0.8686 & 0.9625 & 0.9639 & 0.9647 \\
\textbf{SM-MoH} & \textbf{0.8744} & \textbf{0.8975} & \textbf{0.9062} & \textbf{0.7903} & \textbf{0.8675} & \textbf{0.8727} & \textbf{0.9665} & \textbf{0.9657} & \textbf{0.9658} \\
\bottomrule
\end{tabular}
\label{tab:moh_analysis}
\end{table*}

\subsection{SM-MoH Parameter Tuning}

We investigate the influence of two hyperparameters in the MoH module: the number of total experts (horizontal axis) and the activation ratio, i.e., the proportion of experts selected per input (vertical axis). As shown in Figure~\ref{fig:mohparameters}, each subfigure presents the model's 16-bit mAP on CIFAR-10, ImageNet, and MSCOCO, respectively.

Overall, MoH demonstrates stable performance across a wide range of settings, but appropriate tuning of these parameters can yield noticeable improvements. On ImageNet, the best performance (0.8771 mAP) is achieved when using 64 experts with a 1/4 activation ratio, indicating a balanced trade-off between diversity and sparsity. On CIFAR-10, performance remains consistently high, with the best result (0.9666 mAP) also occurring at 64 experts and a 1/4 ratio. For MSCOCO, the highest mAP (0.7903) is observed when activating 1/8 of 64 experts, suggesting that a smaller number of activated experts may be more effective.

It is also notable that overly low expert counts (e.g., 16) or excessively sparse activation (e.g., 1/8 on CIFAR-10) tend to hurt performance, likely due to insufficient model capacity or representational bottlenecks. These results suggest that MoH benefits from a moderate number of diverse experts, with partial activation to maintain efficiency and specialization.

\begin{figure*}[htbp]
    \centering
    \begin{subfigure}[t]{0.33\textwidth}
        \centering
        \includegraphics[width=\linewidth]{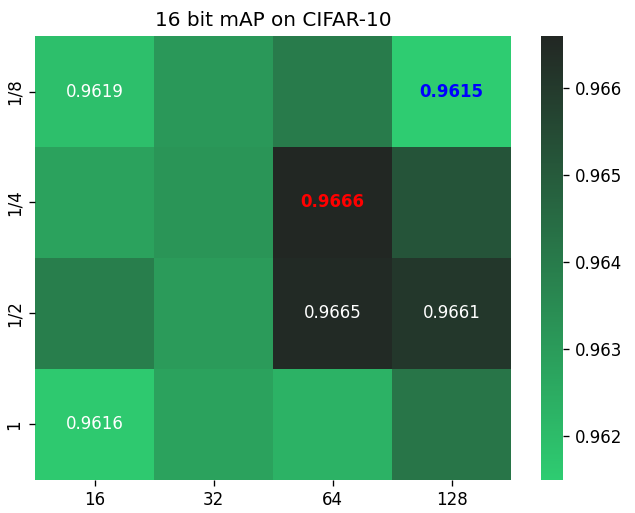}
        \caption{16bit on CIFAR-10}
    \end{subfigure}\hfill
    \begin{subfigure}[t]{0.33\textwidth}
        \centering
        \includegraphics[width=\linewidth]{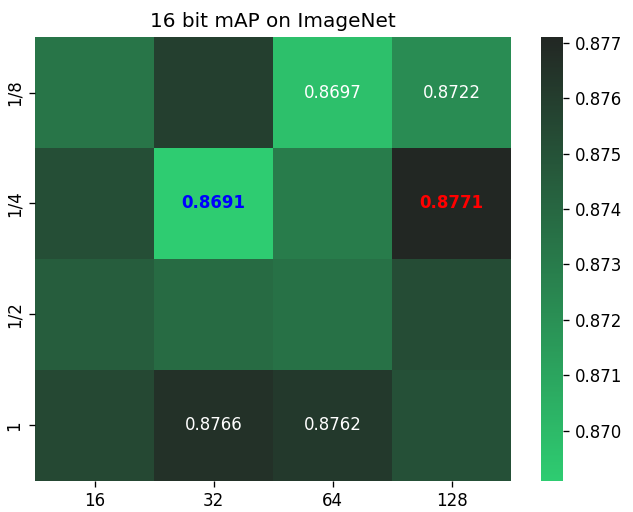}
        \caption{16bit on ImageNet}
    \end{subfigure}\hfill
    \begin{subfigure}[t]{0.33\textwidth}
        \centering
        \includegraphics[width=\linewidth]{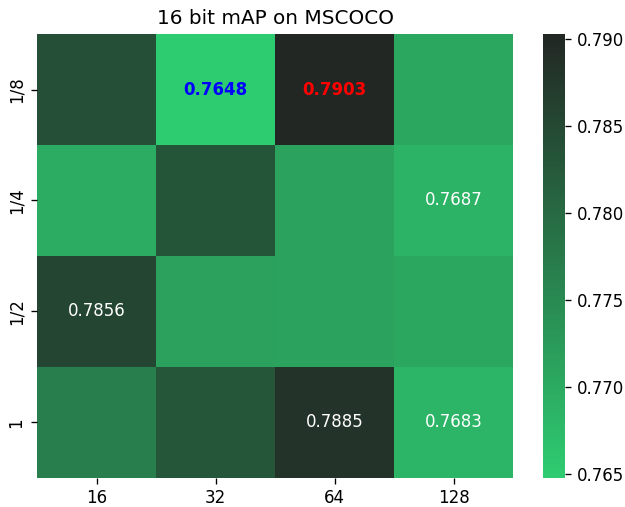}
        \caption{16bit on MSCOCO}
    \end{subfigure}
    \caption{Impact of SM-MoH parameter tuning on model performance at 16-bit code length across ImageNet, MSCOCO, and CIFAR-10 datasets. In heatmaps, darker colors indicate higher values, while lighter colors represent lower values. The maximum and minimum values are highlighted in red and blue, respectively.}
    \label{fig:mohparameters}
\end{figure*}

\subsection{Effectiveness of SM-MoH Design}

SM-MoH is not introduced to boost any individual branch in isolation. Rather, it is specifically designed to facilitate mutual learning by providing a shared expert space and adaptive routing that enhance cross-branch knowledge transfer and alignment between the pointwise and pairwise streams. 

\paragraph{Why using SM-MoH alone can lead to a performance drop.}The reason is that center-based supervision essentially pulls samples toward their class/hash centers, yielding a signal that is relatively global, stable, and low-variance. In contrast, SM-MoH introduces (i) dynamic routing of gating with top-$k$ experts, meaning that samples from the same class may be dispatched to different experts, and (ii) expert specialization, where different experts learn different mappings. When only the center loss is present, there is insufficient constraint to coordinate the outputs across experts. Accordingly, when SM-MoH is used without mutual learning, it may not provide benefits and can even underperform the plain center-based baseline (Table~\ref{tab:ablation}). 

\paragraph{Why mutual learning alone is limited.} Mutual learning without SM-MoH brings only a small gain (Table~\ref{tab:ablation}), because pointwise and pairwise supervision induce representations with different geometric biases, so the two branches often lie in partially misaligned subspaces, limiting effective cross-branch knowledge transfer.

\paragraph{Why the combination works.} When SM-MoH combined with mutual learning, the pairwise branch provides relational constraints among samples, offering a finer-grained and stronger structural signal, while the center-based branch contributes complementary global guidance as shown in Table~\ref{tab:ablation}. Mutual learning further pulls the two branches into a shared hash space, effectively acting as a cross-branch consistency regularizer for the expert routing and outputs. As a result, experts are jointly constrained by the pairwise relational structure and cross-branch alignment, enabling SM-MoH to capture diverse modes/hard cases via multiple experts without sacrificing intra-class consistency.

Therefore, we emphasize that SM-MoH is an interaction-centric component for cross-branch integration, rather than a single-branch booster.

\section{Theory Supplement}

Here we supplement discrepancy elimination theory in the main text: we derive an explicit inequality comparing the single-branch baseline risk \(R_{base}\) with the dual-branch UniHash risk \(R_{UniHash}\), provide a clean non-asymptotic sufficient condition under which \(R_{UniHash}<R_{base}\), and show how these bounds ground our seen/unseen generalization analysis.

\subsection{Risk Comparison: Baseline vs.\ UniHash}
\label{sub:risk}

\paragraph{Risk decompositions.}
Let $R(\cdot)$ be population risk and $\hat{R}(\cdot)$ be empirical risk. 
For single-branch baselines and UniHash, we have
{
\setlength{\abovedisplayskip}{6pt}
\setlength{\belowdisplayskip}{6pt}
\begin{equation}\label{eq:base-bound}
R_{base} \le \hat{R}_{base}
 + C_{stat}^{base}
 + C_{quant}^{base}
 + E_{struct},
\end{equation}
\begin{equation}\label{eq:UniHash-bound}
R_{UniHash} \le \hat{R}_{UniHash}
 + C_{stat}^{UniHash}
 + C_{quant}^{UniHash}
 + C_{cons}.
\end{equation}
}

Here, \(C_{stat}\) denotes capacity; \(C_{quant}\) is a quantization term decaying with code length; \(E_{struct}>0\) is the paradigm-specific structural error of single-branch training; and \(C_{cons}\) is the consistency term from mutual learning.

\paragraph{Capacity \& consistency scalings.}
Under the assumptions already stated in the paper (Lipschitz modules and Top-$k$ sparse routing), there exist constants such that
{%
\setlength{\abovedisplayskip}{6pt}%
\setlength{\belowdisplayskip}{6pt}%
\begin{equation}\label{eq:stat-term}
C_{\text{stat}}^{(\cdot)} = \mathcal{O}\!\Big(\sqrt{\tfrac{k\log m}{n}}\Big),
\end{equation}
\begin{equation}\label{eq:quant-term}
C_{\text{quant}}^{(\cdot)} = \exp(-\Omega(q)),
\end{equation}
\begin{equation}\label{eq:cons-term}
C_{\text{cons}} = \mathcal{O}\!\Big(\tfrac{\tau}{\sqrt{n}}\Big).
\end{equation}
}

where $n$ is sample size, $m$ the number of experts, $k$ the per-sample active experts, and 
$\tau^2=\mathbb{E}\|u^c(x)-u^p(x)\|^2$ measures cross-branch mismatch.

\paragraph{Direct inequality for superiority.}
Subtracting Eq.~\ref{eq:base-bound} from Eq.~\ref{eq:UniHash-bound} yields
{%
\footnotesize
\begin{equation}\label{eq:D4}
\textstyle
\begin{split}
R_{UniHash}\! - R_{base}\!\!
\;\le\;&
\!(\hat{R}_{UniHash}\!-\hat{R}_{base})\!
+(C_{stat}^{UniHash}\!-C_{stat}^{base})\\
&\!\!\!+(C_{quant}^{UniHash}\!-C_{quant}^{base})\!
+(C_{cons}\!-E_{struct}\!).
\end{split}
\end{equation}
}

Under the same backbone and bit settings, 
$C_{stat}^{UniHash}\!\approx C_{stat}^{base}$,
$C_{quant}^{UniHash}\!\approx C_{quant}^{base}$,
and with well-optimized training $\hat{R}_{UniHash}\!\approx\hat{R}_{base}$.
Hence
\begin{equation}
R_{UniHash} - R_{base}
\;\lesssim\;
C_{cons}-E_{struct}.
\label{eq:D5}
\end{equation}

\paragraph{Sufficient condition.}
If after training
\begin{equation}
C_{cons} \;<\; E_{struct} - \varepsilon
\quad(\varepsilon>0),
\label{eq:D6}
\end{equation}
then Eq.~\ref{eq:D5} implies $R_{\text{UniHash}} < R_{\text{base}}-\varepsilon$.
In effect, mutual learning replaces the irreducible single-branch structural error \(E_{struct}\) with a learnable consistency term \(C_{cons}\). Once \(C_{cons}<E_{struct}\), UniHash attains strictly lower population risk.

\paragraph{Asymptotics.}
When $n,q\to\infty$, Eq.~\ref{eq:stat-term}--\ref{eq:cons-term} gives $C_{stat}, C_{cons}\!\to 0$ and $C_{quant}\!\to 0$, so from Eq.~\ref{eq:UniHash-bound} we get 
$R_{UniHash}-\hat{R}_{UniHash}\!\to 0$; 
the single-branch bound  Eq.~\ref{eq:base-bound} retains the constant floor $E_{struct}>0$.
Thus, by eliminating the structural floor that single-branch models retain, UniHash is strictly better asymptotically.

\subsection{Seen/Unseen Generalization}

Let $R_{seen}(\cdot)$ and $R_{unseen}(\cdot)$ be test risks on seen and unseen categories, and 
$\Delta(\cdot)=R_{unseen}-R_{seen}$ be the gap. View seen and unseen samples as drawn from $\mathcal{D}_s,\mathcal{D}_u$ on $(x,y)\!\in\!\mathcal{X}\times\mathcal{Y}$ with a shared $\mathcal{X}$ but label sets $\mathcal{Y}_s$ (seen) and $\mathcal{Y}_u$ (unseen), where only $\mathcal{Y}_s$ is observed in training; risks are expectations under $\mathcal{D}_s$ or $\mathcal{D}_u$. For any hashing function $f$,
{\small
\begin{equation}
R_{unseen}(f)\le R_{seen}(f)+\mathrm{disc}_{\mathcal H}(\mathcal D_s,\mathcal D_u)+\lambda(f).
\label{eq:D7}
\end{equation}
}

where $\mathrm{disc}_{\mathcal{H}}$ measures domain divergence w.r.t.\ hypothesis class $\mathcal{H}$,
and $\lambda(f)$ captures how well the embedding’s global anchors and relative layout transfer to unseen classes.
Thus the gap $\Delta(f)$ is governed by (a) hypothesis complexity and (b) embedding alignment.

\paragraph{How UniHash reduces $\Delta$.}
Mutual learning reduces the right-hand side of Eq.~\ref{eq:D7} along two coupled mechanisms. First, the cross-branch consistency term $C_{cons}=\mathcal{O}(\tau/\sqrt{n})$ enforces $u^c\!\approx u^p$, effectively contracting the admissible hypothesis class and thereby diminishing $\mathrm{disc}_{\mathcal H}(\mathcal D_s,\mathcal D_u)$, since smaller classes are less sensitive to distribution shift. Second, training both branches in a single hash space with shared experts jointly supplies global anchors (pointwise signal) and relational geometry (pairwise signal), which improves semantic alignment on novel categories and lowers $\lambda(f)$; unseen samples are thus mapped onto the structured manifold spanned by seen semantics rather than drifting into uncalibrated regions.

\paragraph{Quantified gap reduction.}
Let $f_{\text{base}}$ be a single-branch optimum and $f_{\text{UniHash}}$ the UniHash solution. Combining Sec.~\ref{sub:risk} with the two effects above yields
{\footnotesize
\begin{equation}
\begin{aligned}
\Delta(f_{\mathrm{UniHash}})-\Delta(f_{\mathrm{base}})
\;\lesssim\;&
C_{\mathrm{cons}}-E_{\mathrm{struct}}
+\mathrm{disc}_{\mathcal H_{\mathrm{UniHash}}}-
\\
&\quad
\!\!\!\!\mathrm{disc}_{\mathcal H_{\mathrm{base}}}
+\lambda(f_{\mathrm{UniHash}})-\lambda(f_{\mathrm{base}}).
\label{eq:D8}
\end{aligned}
\end{equation}
}

If Eq.~\ref{eq:D6} holds and shared experts indeed align the embedding, the RHS is strictly negative:
\emph{$\Delta(f_{\text{UniHash}})<\Delta(f_{\text{base}})$}.
UniHash improves seen$\to$unseen generalization by jointly shrinking the hypothesis-class sensitivity to distribution shift and enhancing semantic alignment. Concretely, the gap reduction in Eq.~\ref{eq:D8} follows from (i) consistency-driven capacity contraction, which lowers $\mathrm{disc}_{\mathcal H}$, and (ii) shared-expert alignment, which reduces $\lambda(f)$. A simple non-asymptotic check (\emph{e.g.}, $C_{cons}<E_{struct}$ on validation) already guarantees a smaller unseen risk relative to single-branch baselines. As code length increases, quantization errors decay and further facilitate these effects, and asymptotically UniHash removes the structural floor retained by single-branch training, yielding strictly better population risk on unseen classes.

\section{Computational and Space Complexity}

We analyze UniHash from both computational and space complexity perspectives. As shown below, inference scales linearly with the batch size and relies on compact binary codes, so the retrieval pipeline remains efficient in practice; in turn, this efficiency offsets the extra complexity introduced by our novel structure, making UniHash suitable for large-scale retrieval.

\subsection{Computational Complexity}
The overall computational complexity is $\mathcal{O}\!\left(N\left(d^{2}+kdq\right)\right)$, where $N$ denotes the batch size, $d$ denotes the output dimensionality of the backbone, $q$ is the length of the binary hash code, and $k$ is the number of experts selected per sample.

A detailed analysis is presented below. Each gating network consists of two MLP layers, resulting in a computational complexity of $\mathcal{O}(d^{2})$, where $d$ denotes the output dimensionality of the backbone. Similarly, each hash expert is composed of two MLP layers, leading to a complexity of $\mathcal{O}(dq)$, where $q$ is the length of the binary hash code. For each sample, the complexity is $\mathcal{O}(d^{2}+kdq)$, where $k$ is the number of experts selected per sample. Therefore, the overall computational complexity is $\mathcal{O}\!\left(N\left(d^{2}+kdq\right)\right)$, where $N$ denotes the batch size.

\subsection{Space Complexity}
The overall space complexity is $\mathcal{O}\!\left(d+m(d+q)\right)$, where $d$ denotes the output dimensionality of the backbone, $q$ is the length of the binary hash code, and $m$ is the total number of experts.

A detailed analysis is presented below. Each gating network consists of two MLP layers, resulting in a space complexity of $\mathcal{O}(d)$, where $d$ denotes the output dimensionality of the backbone. Similarly, each hash expert is composed of two MLP layers, leading to a complexity of $\mathcal{O}(d+q)$, where $q$ is the length of the binary hash code. Since there are $m$ experts, the overall space complexity is $\mathcal{O}\!\left(d+m(d+q)\right)$.

\section{Limitations}

\paragraph{Is this work merely an incremental combination of existing ingredients, rather than a genuinely new contribution?} Our key clarification is that the paper is not optimizing a single paradigm with add-ons, but explicitly targets the structural tension between seen-category accuracy and unseen-category generalization, and proposes a unified dual-branch training formulation to resolve it. In this view, the dual-branch design and the alignment objective are not optional embellishments: they operationalize the complementary failure modes of pointwise vs. pairwise supervision under the same retrieval goal, which is different from prior works that stay within one paradigm and only refine losses or architectures.

\paragraph{Does SM-MoH really contribute on its own, does it only work when paired with mutual learning?}
A strong concern is that SM-MoH alone can degrade performance (e.g., the center-based branch drops notably when SM-MoH is inserted without mutual learning), which could be interpreted as lack of standalone utility or robustness. The paper’s intended claim, however, is that SM-MoH is not designed as a universally beneficial plug-in; it is a structural mechanism to enable cross-paradigm interaction—routing and expert specialization create diversity, while the mutual-learning alignment provides the missing constraint that prevents intra-class dispersion across experts. Therefore, the relevant evaluation is the coupled system: SM-MoH + mutual learning, where ablations show the drop disappears and performance peaks, supporting that the module’s benefit is conditional but purposeful.

\paragraph{Is the practical efficiency worth the added complexity?}
Our design is simplifiable at inference: retrieval only relies on the binary-code generation path, and we can follow the paper’s select-the-better-branch strategy to keep only a single branch for encoding both queries and the database, making inference cost comparable to standard single-branch methods. Moreover, SM-MoH activates only a sparse set of experts via top-k routing, which keeps its marginal compute controlled. The complexity analysis also indicates that most additional overhead comes from the dual-branch interaction during training, which is an offline one-time cost and does not affect online retrieval throughput. Therefore, UniHash remains suitable for large-scale retrieval.

\end{document}